\documentclass[a4paper,11pt]{article}
\pdfoutput=1

\usepackage{amssymb}  
\usepackage[cmex10]{amsmath}
\usepackage{algorithm}
\usepackage{algorithmic}
\usepackage{graphicx,multirow}
\usepackage{array}
\usepackage{amsmath}
\usepackage{amsfonts}
\usepackage{amssymb}
\usepackage{latexsym}
\usepackage{color}
\usepackage{verbatim}
\usepackage{authblk}

\newtheorem{defi}{Definition}

\newtheorem{theo}{Theorem}
\newtheorem{propo}{Proposition}
\newtheorem{lemm}{Lemma}
\newtheorem{coro}{Corollary}
\newtheorem{exam}{Example}

\newenvironment{definition}{\begin{defi} \rm }{\end{defi}}
\newenvironment{theorem}{\begin{theo} \rm }{\end{theo}}

\newenvironment{lemma}{\begin{lemm} \rm }{\end{lemm}}

\newenvironment{proof}{\begin{trivlist} \item[\hspace{\labelsep}\bf Proof:]}{\end{trivlist}}
\usepackage{mathtools}

\newcommand{\M}{\mathcal{M}}

\begin{document}


\title{Verification of Logical Consistency in Robotic Reasoning}

\author{Hongyang~Qu}

\author{Sandor~M.~Veres} 


\affil{Department of Automatic Control and Systems
  Engineering\\  University of Sheffield Sheffield S1 3JD, United
  Kingdom\\
\{h.qu, s.veres\}@sheffield.ac.uk}

\maketitle



\begin{abstract}

Most autonomous robotic agents use logic inference to keep themselves
to safe and permitted behaviour. Given a set of rules, it is important
that the robot is able to establish the consistency between its rules,
its perception-based beliefs, its planned actions and their
consequences.  This paper investigates how a robotic agent can use
model checking to examine the consistency of its rules, beliefs and
actions. A rule set is modelled by a Boolean evolution system with
synchronous semantics, which can be translated into a labelled
transition system (LTS). It is proven that stability and consistency
can be formulated as computation tree logic (CTL) and linear temporal
logic (LTL) properties. Two new algorithms are presented to perform
realtime consistency and stability checks respectively. Their
implementation provides us a computational tool, which can form the
basis of efficient consistency checks on-board robots.

\end{abstract}





\section{Introduction} \label{sec:intro}

A robotic system's decision making is well known to be in need of some hard
decision making at times. A most popular example is Asimov's Laws
\cite{roblaws}, which demonstrate the difficulties to apply logic by robots in practice. A
shortened version of these laws is ``1. A robot may not allow a human
being to come to harm. 2. A robot must obey the orders given to it by
human beings except if the order causes harm to humans. 3. A robot
must protect its own existence as long as such protection does not cause
harm to humans.'' Assuming these, what would happen to the robot's
decision making if a human commands a robot to kill someone, but at the
same time threatens to kill himself if the robot does not obey? In
this example the human introduces a contradiction into the logic of
the robot. To avoid this the robot may have a complex rule base to
provide it with legal and ethical principles and can be equipped by a
meta law which says that ``the robot should not allow itself to be dictated
by communicated conditions which make  its logic contradictory''. In this example 
one could say that in legal terms the suicide will remain the sole
``responsibility'' of the threatening  person who commands the robot.

The problem is not only the imperfection of Asimov's robotic laws or
that an agent programmer can make mistakes. Logical consistency checks
are also needed when the robot's perception-based beliefs are wrong.
The agent can be programmed to re-examine whether its beliefs may need
to be changed as were mistakenly believed to be true or false. This is
not unlike enabling the agent to think like Poirot, Miss Marple or
Sherlock Holmes when they are reassessing their initial beliefs or
impressions. But there are simpler cases: a robot may decide that the
book it sees on the table cannot be Tom's as that one is in his
home. In this paper we address the problem of how a robot can quickly
and efficiently resolve inconsistencies in order to make the right
decisions.

The ability of making fast decisions about logical consistency, and
the robot's ability to detect when inconsistency occurs, is an
important problem for the future of robotics.  It is also of
particular importance for logic-based robot control systems, e.g.,
\cite{ABB05,MW04,SW01,SPBK08,TALF09,TBS06,Vranas08}.  A typical
logic-based robotic system usually contains a belief set, which
provides the basis of reasoning for a robot's
behaviour~\cite{MW04}. An inconsistent belief set could lead to a
wrong plan causing an unexpected result, e.g., an unmanned vehicle can
hit an obstacle, instead of avoiding it, if it mistakenly believes
that any route of avoidance could cause more damage, due to, for instance, mis-perception of the environment. 
Its mis perception could perhaps be corrected if it had been able to combine environmental prior knowledge 
with current sensing.

In a rapidly changing environment Bayesian methods can be used to
identify and track movements of objects and establish functional
relationships, e.g., \cite{MihaylovaCSGPG14}. When faced with balanced
probabilities for two hypothetical and competing relationships in the
robot's environment, it may need to make a decision based
on the application of logic using prior knowledge. Discovery of
logical inconsistency in geometrical and physical relationships in an
environmental model should prompt a robotic agent to revise its
perception model of the world. For instance belief-desire-intention (BDI)
agents should carry out consistency checks in their reasoning cycle in
languages such as $Jason$, $2APL$ and $Jade$
\cite{jason,2apl,jade,nlp12,review2011}. In these systems the agent
programmer should program logical consistency checks and handling of
inconsistencies at design stage of the software.

To topic of fast consistency checking by robots has also implications for legal certification of robots. 
As we humans formulate social
and legal behaviour rules in terms of logical implications, the process is likely to be similar for robots and the
problem of consistent decisions by robots 
 is  an important generic capability. Future legal frameworks for certification of
robots need to take into account verifiable decision making by robots.

Consistency checks on a set of logic rules in propositional logic is a
textbook problem and has been extended to various types of logic
systems in terms of validity, consistency and satisfiability.  For
instance \cite{sathist} provides an authoritative account of the
history of logical consistency checking in a propositional
logic. Relevant methods and algorithms have long been investigated for
database systems and rule-based expert systems, e.g., \cite{MA92}, but
none has been specifically designed for robotics. Query Language
4QL~\cite{4QL} and Boolean Networks (BN)~\cite{Kauffman69} are very
similar to our modelling formalism {\em Boolean evolution
  systems}. The former allows a variable to have four values: $true$,
$false$, $unknown$ and $inconsistent$. The algorithm that computes the
unique well-supported model in~\cite{4QL} can be adapted to check
consistency, but it can only deal with one initial evaluation of
variables at a time. BN was developed for modelling gene regulatory
networks in Biology. In BN, a Boolean variable can only take either
$true$ or $false$, while in our formalism, a variable can be
initialised as $unknown$. Research on BDI reasoning cycles focuses on
runtime detection and resolution of conflicting goals, such
as~\cite{Thangarajah02,MorrealeBFPCCCG06}. No work has been conducted
on complex reasoning process, which will be required by autonomous and
intelligent robots.

For realtime robotic systems it is important
to increase solver efficiency to be able to deal with large search
spaces with complex reasoning process for both
offline and online application.  
In this respect, the use of binary decision diagram (BDD) is
very effective by compressing search space through generating a unique
and succinct representation of a Boolean formula. BDD has been widely
adopted for model checking~\cite{cgp99}, and applied successfully to
verification of large systems. In this paper we adopt the BDD based
symbolic model checking approach~\cite{Burch+92a} to robotics. To our
best knowledge, nothing has been reported on its application on
consistency and stability checking of decisions by robots.

In this paper we propose a fast method for discovery of inconsistency
in a set of logic rules and statements on relationships in a current
world model, past actions, planned actions and behaviour rules of a
robotic agent.  We do not address the problem of how to resolve
logical inconsistency, mainly because we hold the view that, to
eliminate inconsistencies, a robot can efficiently improve its world
model by non-logic based techniques. Such techniques can include
gathering more perception data, active vision, using alternative
action plans or analyzing and deriving spatial temporal models using
probabilities.  If a single new perception predicate or predicate derived by logic rules of the robot
contradicts its otherwise consistent world model, then the robot may
apply a set of logic rules to derive a correction of 
its belief in terms of the predicate. What to derive and analyse for consistency is however a broad topic and lies outside of
the scope of this paper. Here we focus on fast discovery of
inconsistencies which is fundamental for safe operations of autonomous
robots.  With time it should be a key technical part in the process of
legal certification of future autonomous robots.

Our contribution builds on and develops our past efficient state space
generation and parallel computation~\cite{KLQ10} methods further. We
have previously developed various state space reduction techniques for
symbolic model checking via BDDs, such as symmetry
reduction~\cite{CDLQ09a,CDLQ09b} and abstraction~\cite{LQR10}.
The preliminary results of our techniques have been published
in~\cite{QV14}. In this paper we elucidate the setting for which our
techniques are designed and demonstrate their  way of using it
in robotics. We also extend the techniques to deal with a different
semantics and develop a new technique to extract
counterexamples efficiently when the system is inconsistent or
unstable. The counterexamples are useful for system developers to
correct robotic reasoning systems; they can provide guidance on how to improve  the
reasoning process of robots.

We study the efficiency of the agent's ability to examine the
consistency of its beliefs and logic rules and, if inconsistency
occurs, generate counterexamples to the rules which can then be used
by the robot to resolve inconsistency.  Our technique can be used
both by robot programmers at software design stage and by robots when
reasoning. In the former case, system developers can check the logical
consistency of reasoning cycles in agent programs at design stage. For
each inconsistent check, a counterexample can be produced to help
developers understand the source of inconsistency and correct the
program. In the latter case, consistency checks are carried out by the
robots themselves in realtime and counterexamples are examined to
improve reasoning, e.g., bringing in more sensor data to eliminate
ambiguity or bring about alternative decisions about future actions.

In Section~\ref{sec:system} we introduce the problem in a robotic
framework and its characteristics. In Section~\ref{boolevul} Boolean
evolution systems are formally represented. In
Section~\ref{sec:modelling}, we translate Boolean evolution systems
into \emph{transition systems} which are now widely used in the
control systems literature
\cite{TabuadaPappas06:TAC,KloetzerBelta08:TAC}, which provides the
basis of verification. Note that in this paper we abstract robotic
behaviour to propositional logic to be able to cope with computational
complexity of consistency checking. Section~\ref{sec:mc} contains our
results on stability of Boolean evolution systems in terms of CTL and
LTL formulae.  An important result states that stability checking can
be reduced to a reachability problem which only asks for one fixpoint
computation. Similarly, consistency checking can be also converted
into simple fixpoint computation. Section~\ref{sec:case} presents a
case study in a home robotics scenario, which demonstrates the use of
uncertain sensory and communication information and a set of rules to
satisfy. In Section~\ref{sec:exp}, performance comparison between CTL formulae based solutions
and the reachability based algorithms is highlighted and implemented
in the symbolic model checker MCMAS~\cite{MCMAS}. We discuss
stability checking under an alternative semantics of evolution in
Section~\ref{sec:interleaving}. 
We conclude the paper in
Section~\ref{sec:concl}.

\section{Perception clarification and robot logic} \label{sec:system}

Our predicates-based knowledge representation of a robot, which is
derived from sensing events, remembering the past as well as from prediction of a future
environment, is schematically depicted in Fig. \ref{robpreds}. For new
sensory predicates we assume that the robot is able to identify which
are uncertain in a probabilistic sense. The following specific
problems are to be addressed:
\begin{enumerate}
\item Assuming inconsistency occurs, identify which uncertain new
  sensing predicates in $ U_t\subseteq B_t$ can be made certain within
  rules $R^P$ based on physical models.
\item The agent considers a set $A_t\subseteq B_t$ of actions as its options. For
  each action $\alpha_k$ in $A_t$ it simulates a physical model
  over a time horizon and abstracts a set of events $F_t$ for its
  future consequences.
\item It checks if $F_t\subseteq B_t$ and its behaviour rules
  $R^B$ are consistent based on 1) and 2).
\item The set $P_t\subseteq A_t$ of feasible actions $\alpha_k$ in $A_t$ , which are consistent with
  $R^B$, are used by the robot to make a final choice of an action using non-logic based evaluations (for instance
  using planning).
\end{enumerate}

\subsection{Discovering inconsistency}

In Fig.~\ref{robpreds} the diamonds indicate the procedural locations of logical consistency checks, based on 
predicates and sets of rules (logical implications). 
\begin{figure} [h!t!]
\centering{\includegraphics[scale=0.9]{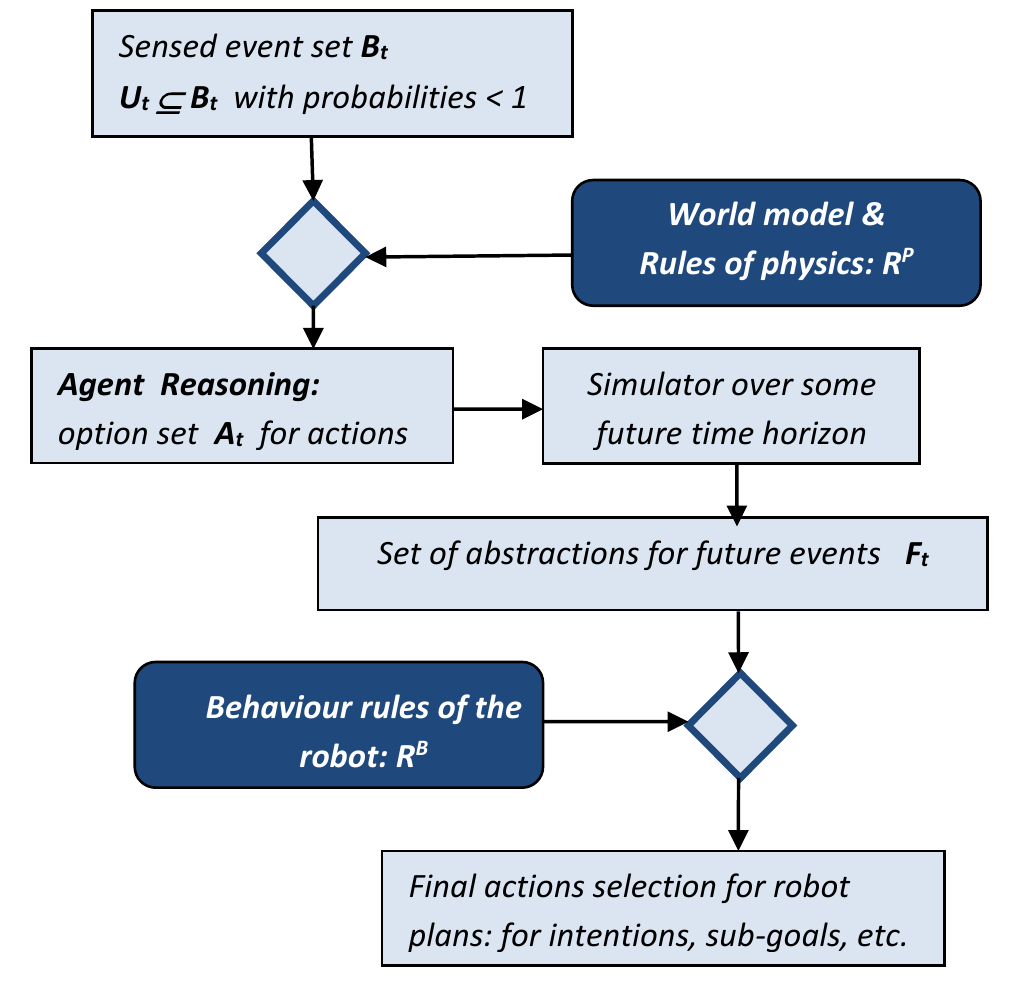}}
\caption{Types of predicates in a robot's reasoning at time \emph{t}. }
\label{robpreds}
\end{figure}
It can however happen that some of the probabilistic sensing of events
remain unresolved based on physical models and associated rules: let
$D_t\subseteq U$ denote the set of undecided perceptions. The robotic
agent needs to check for each of its possible actions what would
happen if various combinations of its uncertain perceptions in $D_t$
were true or false. In safety critical situations a robot cannot take
any action, which could lead to it breaking its rules in some
combination of truth values in $D_t$. Checking this can require
complex consistency checking to be done while the robot interacts with
its environment, hence the efficient methods proposed in this paper
are key to timely decisions by a robot.

This paper is not committed to any particular
type of software architecture. We assume that propositional logic using
a predicate system, which can admit arguments but is equivalent to
propositional logic (for decidability properties), is used in the
robotic software. We also assume that the robot perceives and creates
predicates about environmental events and about its actions periodically
within a \emph{reasoning cycle} performed at an approximately
fixed rate per second.
    
At a given reasoning cycle of the robotic agent, indexed with time $t$, the agent holds a set of predicates
$\mathcal{B}_t\subset \mathcal{B}$ in its memory,   possibly some of these with
negation signs. This means that the predicates in $\mathcal{B}_t$
split into two disjoint sets as $\mathcal{B}_t= \mathcal{B}^{true}_t
\cup \mathcal{B}^{false}_t$  consisting of ones assigned value $true$ while the
rest the Boolean value $false$. Such an assignment of Boolean values
in $\mathcal{B}_t$ is called a valuation of the Boolean variables in
$\mathcal{B}_t$ and denoted by $\overline{\mathcal{B}}_t$. The agent
also has a set of rules at time $t$ denoted by $\mathcal{R}_t=\{r^t_1,
\cdots, r^t_m\}$. The rule set $\mathcal{R}_t$ may contain more
variables than $\mathcal{B}_t$. Those not in $\mathcal{B}_t$ are
unknown to the agent and its logic reasoning is then interested in the
problem of satisfiability of all the logic rules by suitable
assignments to the unknown variables. In the following we will drop
the time index $t$ as we will be interested in the consistency of
logic rules at any time, in view of some Boolean evaluations. The
terms ``variable'' and ``predicate'' will be used interchangeably. Our
primary problem is that the robotic agent has limited time for logical
derivations, when quick response is required, and it needs
to assess the following:
\begin{itemize}
\item[(1)] Are its current evaluations and its rule base consistent in the sense that unknown variables can take on values to satisfy all the rules? 
\item[(2)] Having answered the previous question negatively, can it modify some of its own Boolean evaluations so that its set of predicates becomes consistent with its set of rules?
\end{itemize}

Testing consistency of a set of evaluations can be achieved by checking satisfiability of the conjunction of the evaluations and the rule set,
and obtaining consistent values for unknown variables can be done by
starting to apply the rules until the Boolean evaluation becomes
stable, i.e. the logical value of no variable changes any more. However,
it can be inefficient to use this method as the
number of evaluations may increase exponentially with the number of
variables. 

\subsection{An example of robot reasoning}

By analogy to previous definitions~\cite{lincoln2013,wooldridge2002,veres2011} of AgentSpeak-like architectures for belief-desrie-intention type of robotic agents, we define our reasoning system by a tuple:
\begin{equation}
\label{eq:agent}
\mathcal{R}=\{ \mathcal{F},B,L,\Pi,A\}
\end{equation}
where:
\begin{itemize}
\item 
	$\mathcal{F} = \{p_1,p_2,\ldots,p_{n_p}\}$ is the set of all predicates.
\item
	$B \subset \mathcal{F}$ is the total atomic belief set. The current belief base at time $t$ is defined as $B_t \subset B$. At time $t$ beliefs that are added, deleted or modified are considered \emph{events} and are included in the set $E_t \subset B$, which is called the \emph{Event set}. Events can be either \emph{internal} or \emph{external} depending on whether they are generated from an internal action, in which case are referred to as ``mental notes'', or an external input, in which case are called ``percepts''. 
\item
	$L = R^P\cup R^B= \{l_1,l_2,\ldots\,l_{n_l}\}$ is a set of implication rules.
\item
	$\Pi = \{\pi_1,\pi_2,\ldots,\pi_{n_\pi}\}$ is the set of executable plans or \emph{plans library}. Current applicable plans at time $t$ are part of the subset $\Pi_t \subset \Pi$, this set is also named the \emph{Desire set}. A set $I \subset \Pi$ of intentions is also defined, which contains plans that the agent is committed to execute. 
\item 
	$A = \{a_1,a_2,\ldots,a_{n_a}\} \subset \mathcal{F} \setminus B$ is a set of all available actions. Actions can be either \emph{internal}, when they modify the belief base or generate internal events, or \emph{external}, when they are linked to external functions that operate in the environment.
\end{itemize}

AgentSpeak-like languages, including LISA (Limited Instruction Set Architecture) \cite{ecc16lisa,taros16}, can be fully defined and implemented by listing the following characteristics:
\begin{itemize}
\item \emph{Initial Beliefs}.\\
	The initial beliefs and goals $B_0 \subset \mathcal{F}$ are a set of literals that are automatically copied into the \emph{belief base} $B_t$ (that is the set of current beliefs) when the agent mind is first run.
\item \emph{Initial Actions}.\\
	The initial actions $A_0 \subset A$ are a set of actions that are executed when the agent mind is first run. The actions are generally goals that activate specific plans.
\item \emph{Logic rules}.\\
	A set of logic based implication rules $L=R^P\cup R^B$ describes \emph{theoretical} reasoning about physics and about behaviour rules to redefine the robot's current knowledge about the world and influence its decision on what action to take.
\item \emph{Executable plans}.\\
	A set of \emph{executable plans} or \emph{plan library} $\Pi$. Each plan $\pi_j$ is described in the form:
	\begin{equation}
		p_j : c_j \leftarrow a_1, a_2, \ldots, a_{n_j}
	\end{equation}
	where $p_j \in P_t$ is a \emph{triggering predicate} obtained by consistency in $ U_t \cup F_t  \cup P_t \subset B_t$ and possible valuation for the best choice of  $p_j$ from $P_t$.  Next the $p_j \in P_t$  allows the plan to be retrieved from the plan library whenever it becomes true; $c_j \in B$ is called the \emph{context}, which allows the agent to check the state of the world, described by the current belief set $B_t$, before applying a particular plan;  the $a_1, a_2, \ldots, a_{n_j} \in A$ form a list of actions to be executed.
\end{itemize}
The above list of steps are cyclically repeated to run the reasoning process of a robotic agent.

\section{Boolean evolution systems}
\label{boolevul}

A binary-decision-diagram (BDD) \cite{Bryant-bdd} is a succinct 
representation of a set of Boolean evaluations and, motivated by this,
we examine the possibility of applying symbolic model
checking via BDDs to verify consistency and stability. This way, we
avoid the combinatorial explosion of evaluations. We will show that
BDD based model checking is very efficient for this task to be
carried out in realtime, while the agent needs to give quick responses
to its environment. As agent perception processes 
are often prone to errors in a physical world due to sensor issues or to
unfavourable environmental conditions, this is an important problem of
robotic systems. We present a
formal definition of the consistency checking problems in the next section.  

\begin{definition}[Boolean evolution system]
A Boolean evolution system $BES = \langle \mathcal{B},\mathcal{R}
\rangle$ is composed of a set of Boolean variables
$\mathcal{B}=\{b_1,\cdots, b_n\}$ and a set of evolution rules
$\mathcal{R} = \{r_1, \cdots, r_m\}$ defined over $\mathcal{B}$. A rule
$r_i$ is of the form $g \rightarrow X$, where $g$ is the guard, i.e., a Boolean formula over
$\mathcal{B}$, and $X$ is an assignment that assigns $true$ (``1'') or
$false$ (``0'') to
a Boolean variable $b\in \mathcal{B}$.  For simplicity, we write a rule
of the form $g \rightarrow b:=true$ as $g \rightarrow b$, and write $g
\rightarrow b:=false$ as $g \rightarrow \neg b$. We also group rules
with the same guard into one. For example, 
 two rules $g
\rightarrow b$ and $g \rightarrow c$ can be written as $g \rightarrow b \land c$.
\end{definition}

In practice, the set $\mathcal{B}$ is usually partitioned into 2
subsets: $\mathcal{B}^{known}$ and
$\mathcal{B}^{unknown}$, where variables in the former are
initialized to either $true$ or $false$, and variables in the latter
initialized to $unknown$. Accordingly, the guard
of a rule can be evaluated to $true$, $false$ and $unknown$. The last
case can occur when the guard contains a variable in
$\mathcal{B}^{unknown}$. 

To model a predicates-based knowledge representation and reasoning
system in Fig.~\ref{robpreds} by a BES, we translate each predicate in $B_t$,
action in $A_t$ and future event in $F_t$ into a Boolean variable and
each reasoning rule in $R^P \cup R^B$ into a Boolean formula. In
particular, the uncertain sensing predicates in $U_t\subseteq B_t$
and future events in $F_t$ are
placed in $\mathcal{B}^{unknown}$, and those in $B_t\setminus U_t$ and
actions in $A_t$ are placed in
$\mathcal{B}^{known}$. 

Let $\overline{\mathcal{B}}$ be a valuation of the Boolean variables,
and $\overline{\mathcal{B}}(b)$ the value of variable $b$ in
$\overline{\mathcal{B}}$. 
%
We say that a rule $r\in \mathcal{R}$ is {\em enabled} if
its guard $g$ is evaluated to $true$ on $\overline{\mathcal{B}}$.
The new valuation, after applying the evolution rules  to
$\overline{\mathcal{B}}$, is defined by   {\em synchronous evolution
  semantics} as follows. 

\begin{definition}[Synchronous evolution semantics]
  Let $\mathcal{R}|_{\overline{\mathcal{B}}}\subseteq \mathcal{R}$ be
  the set of rules that are enabled. The new valuation
  $\overline{\mathcal{B}}'$ is the result of simultaneously applying
  all rules in $\mathcal{R}|_{\overline{\mathcal{B}}}$ to
  $\overline{\mathcal{B}}$. That is, every value of $b$ in
  $\overline{\mathcal{B}}'$ is defined as follows.
\begin{equation*} \label{eqn:synchronous}
\overline{\mathcal{B}}'(b) = \left\{
 \begin{array}{l l}
   true & \quad \text{if there exists a rule $g \rightarrow b$ in
     $\mathcal{R}|_{\overline{\mathcal{B}}}$,}\\
   false & \quad \text{if there exists a rule $g \rightarrow \neg b$ in
     $\mathcal{R}|_{\overline{\mathcal{B}}}$,}\\
   \overline{\mathcal{B}}(b) & \quad \text{otherwise.}
 \end{array} \right.
\end{equation*}

\end{definition}

The evolution from $\overline{\mathcal{B}}$ to $\overline{\mathcal{B}}'$
is written as $\overline{\mathcal{B}} \longrightarrow
\overline{\mathcal{B}}'$. We assume that for each valuation, there
exists a non-empty set of enabled rules. 

\begin{definition}[Stability]
A Boolean evolution system is
{\em stable} if from any valuation and applying the rules recursively,
it eventually reaches a valuation $\overline{\mathcal{B}}$ where no
other valuation can be obtained, i.e., $\overline{\mathcal{B}}'
=\overline{\mathcal{B}}$. We say that $\overline{\mathcal{B}}$ is a
{\em stable} valuation, written as $\overline{\mathcal{B}}_s$. 
\end{definition}
Whether
stability happens is decidable by the agent: it requires that two
consecutive steps in the evolution have identical valuations.

\begin{definition}[Inconsistency] \label{def:incon}
Three problems might occur during evolution of a BES: 
\begin{enumerate}
\item
two enabled rules try to update the same Boolean variable with
opposite values at some time;
\item
a variable in $\mathcal{B}^{known}$ is updated to the opposite value
of its initial value at some time.
\item
a variable in $\mathcal{B}^{unknown}$ is updated to the opposite value at some time
after its value has been determined\footnote{The third problem is
  different from the second one because the variables in
  $\mathcal{B}^{unknown}$ are initially set to $unknown$, which can
  be overwritten using the evolution rules.}.
\end{enumerate}
If any of these problem happens, we say that the system is {\em
  inconsistent}. Otherwise, the system is {\em consistent}.
\end{definition}

These problems should
be identified when robotic software is programmed. For instance 
belief-desire-intention rational agent implementations  
apply the logic rules in each reasoning cycle in $Jason$, $2APL$ and
$Jade$ \cite{jason,2apl,jade}. Within one
reasoning cycle, where the input to the variables in
$\mathcal{B}^{known}$ is kept constant. This justifies the second and third problems in Definition~\ref{def:incon}.

{\bf Example 1.}
\begin{equation*}
\begin{split}
 \mathtt{a} & \rightarrow \neg \mathtt{b} \land \mathtt{c} \\
 \neg \mathtt{b} & \rightarrow \neg \mathtt{c}
\end{split}
\end{equation*}

This example demonstrates the inconsistency under synchronous
semantics. For the initial valuation $a=true\land b=c=unknown$, both the first and
second rules are enabled, which makes $b=false$ and $c=true$. In the
next evolution iteration, the second rule sets $c$ to $true$, while
the third one sets $c$ to $false$. Fig.~\ref{example1-cex} illustrates
the evaluation in these evolution iterations.
\begin{figure}[h!]
\centering{\includegraphics[scale=0.7]{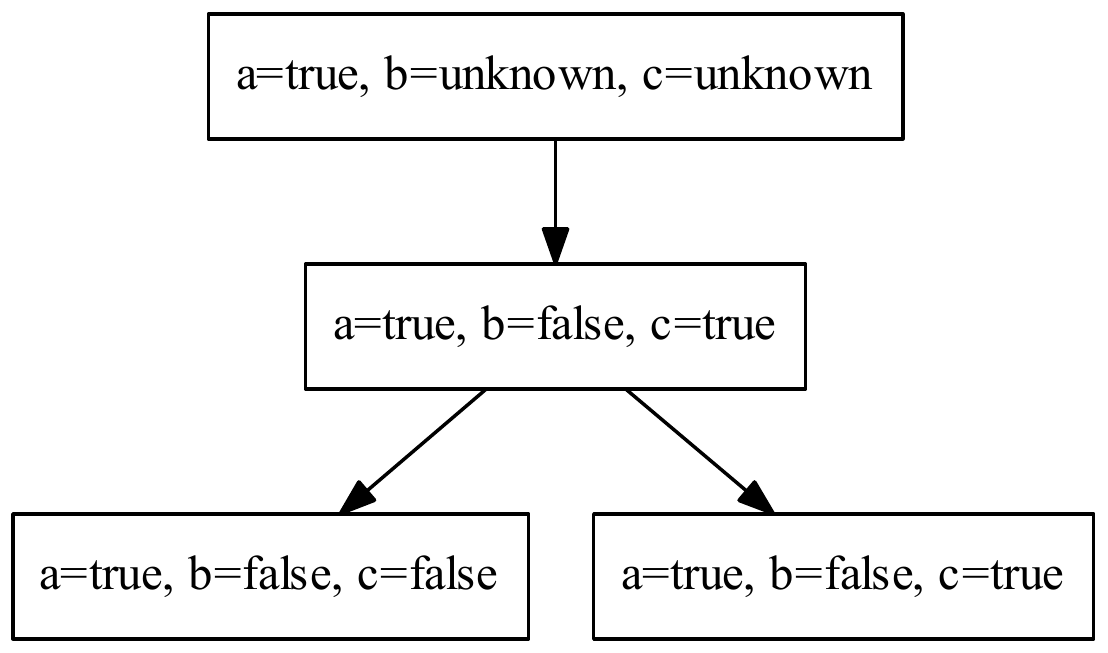}}
\caption{The evolution process showing inconsistency.}
\label{example1-cex}
\end{figure}

The following result can be used to provide a simple algorithm to
solve problem (1) of the agent.

 \begin{theorem} \label{theorem:simple}
   Let $\overline{\mathcal{B}}$ be a Boolean evaluations of 
   variables in the rule set $\mathcal{R}$. Then the following hold.

  If the Boolean evolution system is not stable then
  $\overline{\mathcal{B}}$ and $\mathcal{R}$ are inconsistent which
  the agent can detect from the same evaluation reoccurring during the
  Boolean evolution.
\end{theorem}

\begin{proof}
  If the evolution is not stable, then during the transition steps
  between a recurrence of the same evaluation, some evaluations
  must be different as otherwise the evolution would be stable with
  the evaluation that occurred at two consecutive identical
  evaluations. As an evaluation reoccurs, this means that some
  variable values in that evaluation are changed between the two
  identical occurrences of valuations. Let $a$ be such a variable. The
  logic rules applied, which led to the recurrence of
  the evaluation, have in fact forced $a$ at least once to change
  to its opposite value and later change back. This means that the
  rule set $\mathcal{R}$ for the initial evaluation is inconsistent
  with the rules, i.e. $\mathcal{R}$ is not satisfiable by any
  evaluation which is consistent with the initial evaluation
  $\overline{\mathcal{B}}$. 
\end{proof}

Theorem~\ref{theorem:simple} shows that stability is guaranteed in
consistent systems. For certain systems, however, the inconsistency
conditions in Definition~\ref{def:incon} are considered unnecessarily
strict in that the initial value of {\em known} variables may not be
obtained directly from the environment. Hence, these values can sometimes be incorrect. 
To alleviate this problem, the second and third inconsistency
condition in Definition~\ref{def:incon} can be relaxed. Using this principle, 
we say that the second and the third conditions are
{\em solvable} if the system eventually reaches a stable state by
ignoring these two conditions. This principle makes consistency
and stability checking not straightforward any more:  some rules can correct the evaluations 
of some predicates. 
\vspace{3mm}

{\bf Example 2.}

\begin{equation*}
\begin{split}
 \mathtt{a} &\rightarrow \mathtt{b} \land \mathtt{d}\\
 \mathtt{b} \land \mathtt{d} &\rightarrow \neg \mathtt{c} \land \neg \mathtt{a}\\
 \neg \mathtt{c} \land \mathtt{d} &\rightarrow \neg \mathtt{b}\\
 \neg \mathtt{b} \land \mathtt{d} &\rightarrow \mathtt{c}\\
 \mathtt{c} \land \mathtt{d} &\rightarrow  \mathtt{b}\\
 \mathtt{b} \land \mathtt{c} &\rightarrow \neg \mathtt{d}
\end{split}
\end{equation*}

This example shows a consistent and stable system, where
$\mathcal{B}^{known}=\{a\}$ and $\mathcal{B}^{unknown}=\{b, c,
d\}$. We use a sequence of `$0$', `$1$' and `$?$' to represent states. For
example, the initial state `$0???$' represents $a=false\land b=unknown
\land c=unknown \land d=unknown$.
\begin{itemize}
\item
For valuation $a=false$, the evolution is 
$0???\longrightarrow 0??? \longrightarrow \cdots$.
\item
For valuation $a=true$, we have
$1???\longrightarrow 11?1 \longrightarrow 0101 \longrightarrow 0001 
\longrightarrow 0011 \longrightarrow 0111 \longrightarrow 0100 
\longrightarrow 0100 \longrightarrow \cdots$.
\end{itemize}

\section{Modelling Boolean evolution systems} \label{sec:modelling}


In this section, we describe how to enable model checking to deal with
Boolean evolution systems. First, we introduce {\em transition
  systems}, which are a mathematical formalism that forms the basis of
model checking. Second, we present an example of encoding a Boolean
evolution system under the semantics of transition systems using an
input language of a model checker.

\subsection{Transition systems} \label{subsec:ts}
Model checking is usually performed on transition systems. Here we
present the definition of transition systems and the translation of a
Boolean evolution system into a transition system.

\begin{definition}[Transition system]
A transition system $\mathcal{M}$ is a tuple $\langle S, S_0, T, A$, $H
\rangle$ such that
\begin{itemize}
\item
$S$ is a finite set of states;
\item
$S_0\subseteq S$ is a set of initial states;
\item
$T\subseteq S\times S$ is the transition relation;
\item
$A$ is a set of atomic propositions;
 \item
$H: S\rightarrow 2^{A}$ is a labelling function mapping states
  to the set of atomic propositions $A$. We denote the set of atomic
  propositions valid in state $s$ by $H(s)$.
\end{itemize}
\end{definition}

Let $\overline{S}\subseteq S$ be a set of states. The function
$Image(\overline{S}, T)$ computes the successor states of
$\overline{S}$ under $T$. Formally,
\[Image(\overline{S}, T)=\{s\in S\mid \exists s' \in \overline{S}
\mbox { such that } (s' ,s)\in T\}.\]

Given a Boolean evolution system $BES=\langle \mathcal{B},
\mathcal{R}\rangle$ with $n_1$ unknown variables, i.e.,
$\mathcal{B}^{unknown}=\{b_1, \cdots, b_{n_1}\}$ and $n_2$ known
variables, i.e., $\mathcal{B}^{known}=\{b_{n_1+1}, \cdots,
b_{n_1+n_2}\}$, let $A= \{\mathtt{B}_1, \cdots,
\mathtt{B}_{n_1},\mathtt{B}_{n_1+1}, \cdots, \mathtt{B}_{n_1+n_2}\}$
$\cup\{\mathtt{D}_1, \cdots, \mathtt{D}_{n_1},$ $\mathtt{D}_{n_1+1},
\cdots,\mathtt{D}_{n_1+n_2}\}\cup\{\mathtt{K}_{n_1+1}, \cdots,
\mathtt{K}_{n_1+n_2}\}$, where $\mathtt{B}_i$ is an atomic proposition
representing that a variable $b_i\in \mathcal{B}^{unknown}\cup
\mathcal{B}^{known}$ is $true$, $\mathtt{D}_i$ representing that $b_i$
is $false$, and $\mathtt{K}_j$ representing that an unknown variable
$b_j\in \mathcal{B}^{unknown}$ has value $unknown$. A transition
system (TS) $\mathcal{M}$ can be generated from $BES$ as follows.
\begin{enumerate}
\item
$S$ is composed of all $3^{n_1}\times 2^{n_2}$ valuation of
  $\mathcal{B}$. 
\item 
$S_0$ is composed of $2^{n_2}$ valuations, where variables in
  $\mathcal{B}^{known}$ can take either $true$ or $false$, and
  variables in $\mathcal{B}^{unknown}$ take $unknown$.
\item A transition $(\overline{\mathcal{B}}, \overline{\mathcal{B}}')
  \in T$ iff $(\overline{\mathcal{B}} \longrightarrow
  \overline{\mathcal{B}}')$. In the presence of inconsistent update of
  some Boolean variables, the successor valuation is chosen randomly,
  which results in multiple successor states. For example,
  consider a valuation $s$ from where a Boolean variable $a$ can be
  updated to $true$ by rule $r_1$ and $false$ by rule $r_2$. In the
  transition system, $s$ has two successor states, i.e, valuation: one
  state contains $a=true$ and the other contains $a=false$. If there
  are $k$ Boolean variables that are updated inconsistently in $s$,
  then $s$ has $2^k$ successor states.
\item 
$H(\overline{\mathcal{B}})$ is defined such that for each variable
  $b_i\in \mathcal{B}^{unknown}\cup \mathcal{B}^{known}$,
  $\mathtt{B}_i\in H(\overline{\mathcal{B}})$ iff $b_i$ is evaluated
  as $true$, $\mathtt{D}_i\in H(\overline{\mathcal{B}})$ iff $b_i$ is evaluated
  as $false$, and for each variable $b_j\in \mathcal{B}^{unknown}$,
  $\mathtt{K}_j\in H(\overline{\mathcal{B}})$ iff $b_j$ is evaluated
  to $true$.
\end{enumerate}
Note that all possible input values of variables in
$\mathcal{B}^{unknown}$ are captured by $S_0$, i.e., each possible
valuation of $\mathcal{B}^{unknown}$ is encoded into an initial state
in $S_0$.

The set of states and the transition relation in a transition system
can be illustrated as a direct graph, where a transition $(s_1,
s_2)\in T$ is represented by an arrow from $s_1$ to
$s_2$. Fig.~\ref{example1-lts} shows the directed graph for Example 1
in Section~\ref{sec:system}.
\begin{figure}[h!]
\centering{\includegraphics[scale=0.7]{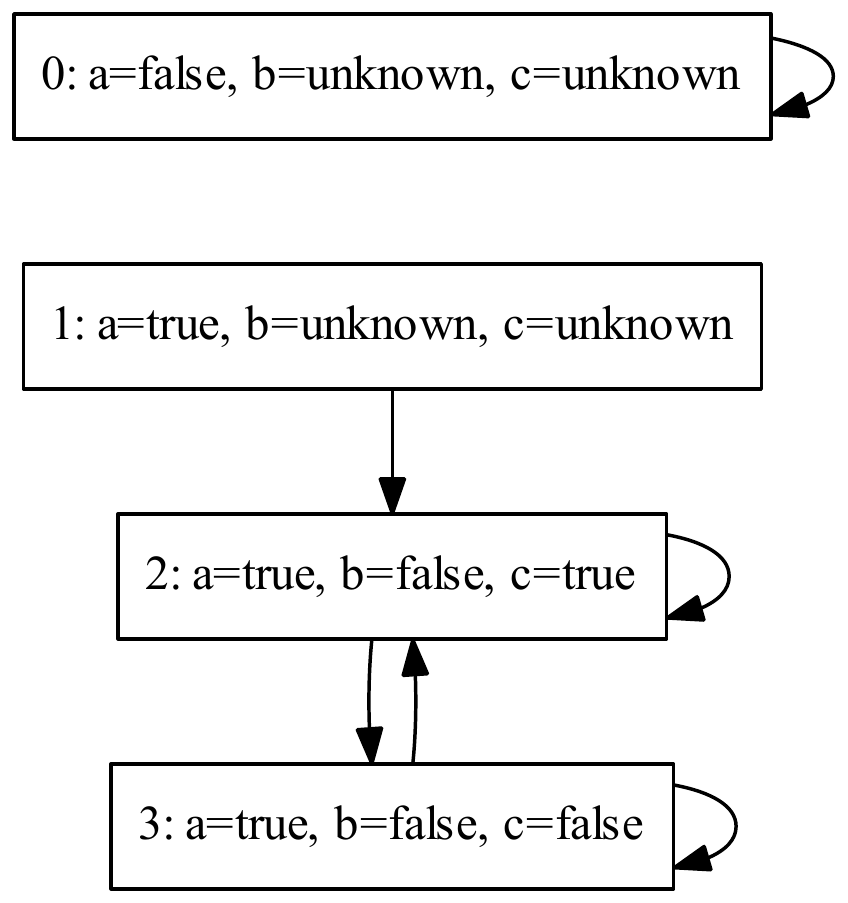}}
\caption{The transition system for Example 1.}
\label{example1-lts}
\end{figure}

\subsection{Implementation} \label{subsec:ispl}
A Boolean evolution system can be written as a program in the input
language of a symbolic model checker, such as NuSMV~\cite{NuSMV}. The
program is then parsed by the model checker to build a transition
system. In this section, we show how to model a Boolean evolution
system by an ISPL (Interpreted System Programming Language)
\cite{MCMAS} program, inspired by the Interpreted System semantics
\cite{fhmv}, and the corresponding transition system can be generated
by the model checker MCMAS \cite{MCMAS}. We use Example 1 to
illustrate how to construct an ISPL program from the Boolean evolution
system $BES = \langle
\mathcal{B}^{unknown}\cup\mathcal{B}^{known},\mathcal{R} \rangle$.

An ISPL program contains a set of agents, a set of atomic
propositions, an expression representing
the initial states and a set of logic formulas representing the
specification of the system. The structure of the program is as
follows :
\begin{verbatim}
Agent 1 ... end Agent
...
Agent n ... end Agent
Evaluation ... end Evaluation
InitStates ... end InitStates
Formulae ... end Formulae
\end{verbatim}
where atomic propositions are defined in the section ``Evaluation''
and the initial states defined in ``InitStates''. 
Each agent is composed of a set of
{\em program variables}, a set of {\em actions} that the agent can execute, a
{\em protocol} and an {\em evolution function}. Each agent has a set of local
states that are encoded by its program variables: each valuation of the
variables is a local state. Its protocol defines a set of enabled
actions for each local state, and its evolution function specifies the
transition relation among its local states. The structure of an
agent $M$ is below:
\begin{verbatim}
Agent M
  Vars: ... end Vars
  Actions = {...};
  Protocol: ... end Protocol
  Evolution: ... end Evolution
end Agent
\end{verbatim}
To encode a BES into an
ISPL program, we only need one agent, and this agent has only one
action, which is enabled in every local state. In the rest of this
section, we give  details of the construction of the ISPL
program. The definition of actions and protocol is omitted as
they do not affect the translation of the BES. 

\begin{enumerate}
\item
As explained before, we do not directly list all states in the state space $S$ of the
corresponding transition system. Instead, we define program variables
to match variables in $BES$. Each variable in $\mathcal{B}^{known}$ is
translated into a Boolean variable in ISPL and each variable in
$\mathcal{B}^{unknown}$ into an enumerated variable with three values
$True$, $False$ and $Unknown$. The corresponding ISPL code for Example 1
is as follows.
\begin{verbatim}
Vars:
  a: boolean;
  b: {True, False, Unknown};
  c: {True, False, Unknown};
end Vars
\end{verbatim}

\item Each evolution rule is translated into a guarded transition
  ``$c$ \verb+if+ $g$'' in ISPL, where guard $g$ is a Boolean expression
  over variables, and $c$ is a set of assignments. Indeed, the
  semantics of a guarded transition matches exactly that of an
  evolution rule. The rules in Example 1 are translated into the ISPL code below.
\begin{verbatim}
Evolution:
  b=False if a=true;    	
  c=True if a=true;
  b=False if c=False;
end Evolution
\end{verbatim}

\item
As each variable in $\mathcal{B}^{unknown}$ in $BES$ is initialized to
{\em unknown}, we need to specify this in the initial state section
$IniStates$ in an ISPL program. The following code is generated for
Example 1.
\begin{verbatim}
InitStates
  M.b=Unknown and M.c=Unknown; 
end InitStates
\end{verbatim}
Note that \verb+M+ is the name of the agent, which encapsulates the variables and transitions, and \verb+M.x+ refers to the variable
\verb+x+ in \verb+M+.
\item
An atomic proposition in ISPL is of the form ``$x$ \verb+if+ $g$'',
where $x$ is the name of the atomic proposition, and $g$ is a Boolean
expression that defines the set of states $x$ holds. That is, $x$
holds in any state whose corresponding valuation satisfies $g$. The
ISPL code for Example 1 is below.
\begin{verbatim}
Evaluation
  a_true if M.a=true;
  a_false if M.a=false;
  b_true if M.b=True;
  b_false if M.b=False;
  b_unknown if M.b=Unknown;
  c_true if M.c=True;
  c_false if M.c=False;
  c_unknown if M.c=Unknown;
end Evaluation
\end{verbatim}
\end{enumerate}

The above construction steps suggests that
a compiler can be produced without difficulties to automatically
generated ISPL code from a given Boolean evolution system. 

Although we have shown the possibility of coding a Boolean evolution
system in ISPL, we would like to emphasize that compilers for other
symbolic model checkers can also be constructed when necessary. For
example, the semantics of the input language of NuSMV is similar to
that of ISPL in this setting as we do not use the
capability of specifying actions and protocols in ISPL.



\section{Stability and inconsistency check} \label{sec:mc} 

Computation Tree Logic (CTL) \cite{ClarkeES86} and Linear time Temporal Logic (LTL) \cite{Pnueli77}  are the most popular logics
adopted in verification of transition systems to specify properties
that a system under investigation may possess. CTL is a branching time
logic, which considers all possibilities of future behaviour, while
LTL only deals with one
possible future behaviour at a time.
In this section, we use CTL
to formulate stability and inconsistency checks due to the efficient
implementation of CTL model checking. But we also discuss
the application of LTL when possible.

\subsection{CTL and LTL}
LTL can be specified by the following grammar \cite{cgp99}:
\[ \varphi ::= p \mid \neg \varphi \mid \varphi\land \varphi \mid
\bigcirc \varphi \mid \Box \varphi \mid \Diamond \varphi \mid
\varphi~\mathcal{U}~\varphi\]

CTL on the other hand is given by the extended grammar \cite{cgp99}:
\begin{equation*}
\begin{split}
 \varphi ::= & \; p \mid \neg \varphi \mid \varphi\land \varphi \mid
EX \varphi \mid EG \varphi \mid EF \varphi \mid
E(\varphi~\mathcal{U}\varphi)\mid \\
& AX \varphi \mid AG \varphi \mid AF \varphi \mid
A(\varphi~\mathcal{U}\varphi)
\end{split}
\end{equation*}

Both CTL and LTL are defined over paths in a transition system. Given
a transition system $\mathcal{M} = \langle S, S_0, T, A, H \rangle$, 
a path $\rho=s_0s_1\ldots s_k$ is a (finite or infinite) sequence of states such that
for each pair of adjacent states, there exists a transition in the
system, i.e., $s_i\in S$ for all $0\le i\le k$ and $(s_j, s_{j+1})\in
T$ for all $0\le j < k$. We denote the $i$-th state in the path
$\rho$, i.e., $s_i$, by $\rho(i)$. The satisfaction of CTL and LTL in
$\mathcal{M}$ is defined as follows.

\begin{definition}[Satisfaction of CTL] \label{CTL_sat}
Given a transition system $\M = \langle S, S_0, T, A, H \rangle$ and a
state $s\in S$, the satisfaction for a CTL formula $\varphi$ at state $s$ in
$\M$, denoted by $ s \models \varphi$, is recursively defined as follows.
\begin{itemize}
\item $ s\models p$ iff $p\in H(s)$;
\item $ s\models \neg \varphi$ iff it is not the case that $ s
  \models\varphi$;
\item $ s\models \varphi_1 \land \varphi_2$ iff $ s\models \varphi_1$ and
  $ s\models \varphi_2$;

\item $ s\models EX \varphi$ iff there exists a path $\rho$ starting at
  $s$ such that $  \rho(1) \models \varphi$.
\item $ s\models EG \varphi$ iff there exists a path $\rho$ starting at
  $s$ such that $  \rho(i) \models \varphi$ for all $i\ge 0$;
\item $ s \models EF \varphi$ iff there exists a path $\rho$
  starting at $s$ such that for some $i\ge 0$, $  \rho(i) \models
  \varphi$;
\item $ s \models E(\varphi_1 U \varphi_2)$ iff there exists a path $\rho$
  starting at $s$ such that for some $i\ge 0$, $  \rho(i) \models
  \varphi_2$ and $  \rho(j) \models \varphi_1$ for all $0\le j<i$;
\item $ s\models AX \varphi$ iff for all paths $\rho$ starting at
  $s$, we have $  \rho(1) \models \varphi$.
\item $ s\models AG \varphi$ iff for all paths $\rho$ starting at
  $s$, we have $  \rho(i) \models \varphi$ for all $i\ge 0$;
\item $ s \models AF \varphi$ iff for all paths $\rho$
  starting at $s$, there exists $i\ge 0$ such that $  \rho(i) \models
  \varphi$;
\item $ s \models A(\varphi_1 U \varphi_2)$ iff for all paths $\rho$
  starting at $s$, there exists $i\ge 0$ such that $  \rho(i) \models
  \varphi_2$ and $  \rho(j) \models \varphi_1$ for all $0\le j<i$;
\end{itemize}
\end{definition}

\begin{definition}[Satisfaction of LTL] \label{LTL_sat}
Given a transition system $\M = \langle S, S_0, T, A, H \rangle$ and a
state $s\in S$, the satisfaction for a LTL formula $\varphi$ at state $s$ in
$\M$, denoted $ s \models \varphi$, is recursively defined as follows.
\begin{itemize}
\item $ s\models p$ iff $p\in H(s)$;
\item $ s\models \neg \varphi$ iff it is not the case that $ s
  \models\varphi$;
\item $ s\models \varphi_1 \land \varphi_2$ iff $ s\models \varphi_1$ and
  $ s\models \varphi_2$;

\item $ s\models \bigcirc \varphi$ iff for all paths $\rho$ starting at
  $s$, we have $  \rho(1) \models \varphi$.
\item $ s\models \Box \varphi$ iff for all paths $\rho$ starting at
  $s$, we have $  \rho(i) \models \varphi$ for all $i\ge 0$;
\item $ s \models \Diamond \varphi$ iff for all paths $\rho$
  starting at $s$, there exists $i\ge 0$ such that $  \rho(i) \models
  \varphi$;
\item $ s \models \varphi_1 U \varphi_2$ iff for all paths $\rho$
  starting at $s$, there exists $i\ge 0$ such that $  \rho(i) \models
  \varphi_2$ and $  \rho(j) \models \varphi_1$ for all $0\le j<i$;
\end{itemize}

\end{definition}

%
When we verify whether a CTL/LTL formula $\varphi$ holds on a model, we check if
this formula is satisfied by all initial states, denoted by $\M
\models \varphi$. In
particular, when we say that an LTL $\varphi$ holds in the model,
every path from every initial state has to satisfy $\varphi$. More
details of CTL and LTL, as well as the difference between them, can be
found in~\cite{cgp99}.

\subsection{Formulation of stability and inconsistency  by logic
  formulae} \label{subsec:logic}



\begin{lemma} \label{thm:incons}
The first category of inconsistency can be checked by the following CTL formula
\begin{equation} \label{eqn:incons}
AG(\neg(EX \mathtt{B}_1 \land EX \mathtt{D}_1) \land \cdots \land \neg(EX \mathtt{B}_n \land EX \mathtt{D}_n)). 
\end{equation}
\end{lemma}

\begin{proof}
  If a system is inconsistent due to the first case, then there must exist a state that has
  two successor states such that a variable is evaluated to {\em true} in
  one successor state, and to {\em false} in the other. The CTL formula $EX
  \mathtt{B}_i \land EX \mathtt{D}_i$ captures this scenario for
  variable $b_i$. The negation $\neg(\ldots)$ excludes the occurrence
  of inconsistency caused by $b_i$. Operator $AG$ guarantees that
  inconsistency does not occur in any states. Note that it is not
  necessary to consider a case like $EX \mathtt{K}_i \land EX
  \mathtt{B}_i\land EX
  \mathtt{D}_i$ for an {\em unknown} variable because it cannot be
  assigned to $unknown$ during evolution. 
\end{proof}
If the above formulae are evaluated to true, then the Boolean
evolution system is consistent. Note that this category of
inconsistency cannot be checked by an LTL formula because LTL can only
specify linear properties. However, a small modification would make
LTL work again on checking consistency.
Lemma~\ref{thm:incons} searches for occurrences of inconsistency by
checking if two opposite values of a variable can be reached from one
state. The following theorem focuses on looking for such a state to
perform consistency checks.

\begin{theorem}\label{thm:incons1}
Checking the first category of inconsistency can be transformed into a
reachability problem as follows.
\begin{enumerate}
\item
For each pair of rules $g_1\rightarrow X_1$ and $g_2\rightarrow X_2$,
check if $X_1$ and $X_2$ assign opposite values to the same Boolean variable.
If the answer is yes and $g_1\land g_2\neq false$, then we add a new
atomic proposition $C$ that holds in states satisfying $g_1\land g_2$.
\item
Let $\mathcal{C}=\{C_1,\ldots, C_m\}$ be the set of all newly added
propositions in the previous step.
The first category of inconsistency can be identified by the 
CTL formula 
\begin{equation} \label{eqn:incons1_ctl}
\neg EF(C_1\lor \cdots \lor C_m)
\end{equation}
or the LTL formula
\begin{equation} \label{eqn:incons1_ltl}
\neg \Diamond(C_1\lor \cdots \lor C_m)
\end{equation}
\end{enumerate}
\end{theorem}
The system is consistent if Formula~\ref{eqn:incons1_ctl} or
\ref{eqn:incons1_ltl} is true.

\begin{proof}
If a state satisfies a proposition $C_i$, which is constructed as
$g_{j_1}\land g_{j_2}$, then both guards $g_{j_1}$ and $g_{j_2}$ are
satisfied in the state. Thus, the corresponding rules
$g_{j_1}\rightarrow X_{j_1}$ and $g_{j_2}\rightarrow X_{j_2}$ are
enabled in the state. As these two rules set opposite values to a
variable, inconsistency occurs in this state if it is reachable from
an initial state. $\mathcal{C}$ captures
all states where inconsistency could happen, and $EF$ and $\Diamond$
examine if any of these states can be reached from an initial
state. Note that $\neg EF(C_1\lor \cdots \lor C_m)\equiv\neg(EF\;C_1\lor \cdots \lor
EF\;C_m)$ and $\neg \Diamond(C_1\lor \cdots \lor C_m)\equiv\neg
(\Diamond C_1\lor \cdots \lor\Diamond C_m)$. 
\end{proof}

Although the second and the third cases are not needed in the relaxed
inconsistency conditions, we still present the temporal logic
properties for checking them.
\begin{lemma}
The second category of inconsistency can be checked by the following CTL formula
\begin{equation} \label{eqn:incons2_ctl}
\begin{split}
AG(&\neg(\mathtt{B}_{n_1+1} \land EX \mathtt{D}_{n_1+1}) \land
\neg(\mathtt{D}_{n_1+1} \land EX \mathtt{B}_{n_1+1}) \land \\
&\cdots \land
\neg(\mathtt{B}_n \land EX \mathtt{D}_n) \land 
\neg(\mathtt{D}_n \land EX \mathtt{B}_n)). 
\end{split}
\end{equation}
 or LTL
formula
\begin{equation} \label{eqn:incons2_ltl}
\begin{split}
\Box(&\neg(\mathtt{B}_{n_1+1} \land \bigcirc \mathtt{D}_{n_1+1}) \land
\neg(\mathtt{D}_{n_1+1} \land \bigcirc \mathtt{B}_{n_1+1}) \land \\
&\cdots \land
\neg(\mathtt{B}_n \land \bigcirc \mathtt{D}_n)  \land 
\neg(\mathtt{D}_n \land \bigcirc \mathtt{B}_n)). 
\end{split}
\end{equation}
\end{lemma}

\begin{proof}
  If this case occurs, then there must exist a state $s$ that has
  a successor state $s'$ such that a variable is evaluated to $true$
  in $s$ and $false$ in $s'$, or $false$ in $s$ and $true$ in $s'$.
  The CTL formulas $\mathtt{B}_i \land EX \mathtt{D}_i$ and
  $\mathtt{D}_i \land EX \mathtt{B}_i$ ($n_1+1 \le i \le n$) capture this scenario for
  variable $b_i$. The LTL formulas $\mathtt{B}_i \land \bigcirc 
  \mathtt{D}_i$ and $\mathtt{D}_i \land \bigcirc \mathtt{B}_i$
  have the same effect. The negation $\neg(\ldots)$ excludes the occurrence
  of inconsistency caused by $b_i$. Operator $AG$ (or $\Box$) guarantees that
  inconsistency does not occur in any states. 
\end{proof}

The third category of inconsistency can be checked in the same way
over the {\em unknown} variables. 

\begin{lemma} \label{thm:stability}
The stability problem can be checked by the following CTL formula
\begin{equation} \label{eqn:ctl}
\begin{split}
AF (&(AG\;\mathtt{B}_1 \lor AG \;\mathtt{D}_1 \lor AG\;\mathtt{K}_1) \land \cdots \land \\
& (AG\; \mathtt{B}_{n_1} \lor AG \;\mathtt{D}_{n_1} \lor AG\;\mathtt{K}_{n_1})\land\\
& (AG\;\mathtt{B}_{n_1+1} \lor AG \;\mathtt{D}_{n_1+1}) \land \cdots \land \\
& (AG\; \mathtt{B}_{n} \lor AG \;\mathtt{D}_{n})
). 
\end{split}
\end{equation}

or LTL formula 
\begin{equation} \label{eqn:ltl}
\begin{split}
\Diamond (&(\Box \mathtt{B}_1 \lor \Box \mathtt{D}_1 \lor \Box \mathtt{K}_1) \land \cdots
\land \\
& (\Box \mathtt{B}_{n_1} \lor \Box \mathtt{D}_{n_1} \lor \Box \mathtt{K}_{n_1})\land  \\
& (\Box \mathtt{B}_{n_1+1} \lor \Box \mathtt{D}_{n_1+1}) \land \cdots
\land \\
& (\Box \mathtt{B}_{n} \lor \Box \mathtt{D}_{n}) )
\end{split}
\end{equation}
\end{lemma}
If the above LTL or CTL formula is evaluated to true, then the Boolean evolution system is stable.

\begin{proof}
  In a stable system, every path leads to a stable state, where no
  {\em unknown} variable will change its value any more. Therefore,
  one of three cases $\Box \mathtt{B}_i$, $\Box \mathtt{D}_i$ or
  $\Box \mathtt{K}_i$ for {\em unknown} variable $b_i$ holds in the
  stable state. The last case means that the {\em unknown} variable
  remains $unknown$ during the evolution. The {\em known} variables
  cannot take value {\em unknown}. Thus, we do not need to consider
  them being $unknown$ in the LTL formula. The operator $\Diamond$ specifies that this
  stable state will be reached eventually.  The CTL formula can be
  reasoned in a similar way. 
\end{proof}

\subsection{Efficient algorithms for stability and inconsistency check}

Although Theorem~\ref{thm:incons1} provides a simpler CTL/LTL formula
than Lemma~\ref{thm:incons}, in practice, it can be improved
further. In order to check if the formula $EF \varphi$ is satisfied in
a transition system $\mathcal{M} = \langle S, S_0, T, A, H \rangle$ by
symbolic model checking, we need to compute the set of states
satisfying the formula $EF \varphi$ using a fixed-point
computation. Let $SAT(\varphi)$ represent the set of states satisfying
$\varphi$. The fixed-point computation begins with a set
$X_0=SAT(\varphi)$ and computes a sequence of sets such that
$X_0\subseteq X_1\subseteq \cdots X_{n} \subseteq X_{n+1}$ until
$X_n=X_{n+1}$. The detailed algorithm is presented in
Algorithm~\ref{algo:ef}.
  
\begin{algorithm} \caption{Compute $SAT(EF \varphi)$} \label{algo:ef}
\begin{algorithmic}[1]
\STATE{$X \coloneqq SAT(\varphi)$; $Y \coloneqq \emptyset$}

\WHILE {$Y\not= X$}   
  \STATE{$Y \coloneqq X$; $X \coloneqq X\cup \{s\in S\mid \exists s'\in X \mbox{ such that }
    (s, s')\in T\}$;}
\ENDWHILE
\STATE{\bf{return} $X$}
\end{algorithmic}
\end{algorithm}

Algorithm~\ref{algo:ef} could be time-consuming for a large
system. Fortunately, we can utilise a
characteristic of model checking to avoid the problem of checking $EF$.
A model checker generates only reachable states, which can be
reached from initial states, and perform model checking algorithms on
the reachable states. To identify inconsistency, showing the existence
of a reachable state with two conflict successor states is
sufficient. As model checkers only work on reachable states, the
existence of a bad state can be converted into non-emptiness of the
set of states satisfying $\mathcal{C}=\{C_1,\ldots, C_m\}$ defined in
Theorem~\ref{thm:incons1}, returned by a model checker. Therefore,
the fixed-point computation for $EF$ can be
avoided. Indeed, checking existence of bad states can be integrated
into the process of generation of reachable state space. Once a bad
state is found, the process can be aborted to give fast feedback to
the programmer.

For a large system, the CTL formula specified in
Lemma~\ref{thm:stability} involves a large number of conjunction
clauses $AG\;\mathtt{B}_i \lor AG \;\mathtt{D}_i \lor
AG\;\mathtt{K}_i$ or $AG\;\mathtt{B}_j \lor AG \;\mathtt{D}_j,$ and
each $AG$ requires a computational
expensive fixed-point computation, as $AG \varphi = \neg EF (\neg \varphi)$. Therefore, model checking this formula could be time
consuming. The following theorem tells us that stability checking can
be reduced to a reachability problem, which only asks for one fixed-point
computation.

\begin{theorem}\label{thm:reachability}
Stability in a consistent system $\M=\langle S, S_0, T, A, H \rangle$
can be checked in the following three steps.
\begin{enumerate}
\item
Find the set $X$ of states that only have self-loop transitions, i.e.,
$X=\{s\in S\mid \forall s' \mbox{ such that } (s,s')\in T \mbox{
  implies } s'=s\}$;
\item
Find the set $Y$ of states that can reach states in $X$;
\item
Check if $S_0\subseteq Y$. If the answer is yes, then the system is stable if it
is consistent.
\end{enumerate}
\end{theorem}

\begin{proof}
  From the definition of stability, we know that a stable valuation
  corresponds to a state that only has self-loops, and vice versa. In
  a consistent system, a state cannot enter a non-stable loop if it
  can reach a stable state. Otherwise, there exists a state that has
  two successor states, which contradicts the assumption that the
  system is consistent. Step 3 checks if there exists an initial state
  that cannot reach a stable state. The existence of such a state
  means that the system contains a non-stable loop. 
\end{proof}

\subsection{Implementation}
For instance the CUDD library~\cite{Bryant-bdd} can be used to manipulate BDDs in
MCMAS. The first step can be implemented using the function
``Cudd\_Xeqy'' in CUDD, which constructs a BDD for the function
$x=y$ for two sets of BDD variables $x$ and $y$. When applied to the
transition relation in the system, this function simply enforces that
the successor state of a state $s$ is $s$ itself, i.e., a self-loop.
The second step can be achieved by the classic
model checking algorithm for $EF$. The third step is done by checking
if $S_0-Y$ is a zero BDD, which means that the result from the set
subtraction $S_0-Y$ is empty. Therefore, this algorithm runs more
efficiently than model checking the long formula in
Lemma~\ref{thm:stability}. In practice, this stability check can be
combined with consistency checks. During the generation of the
reachable state space, we check if the system is consistent using
Theorem~\ref{thm:incons1}. If the generation is not aborted due to the
occurrence of inconsistent states, then a stability check is  
executed. 





\subsection{Counterexample generation}
A common question asked after a formula is model checked is
whether a witness execution or counterexample can be generated to
facilitate deep understanding of why the formula holds or does not hold
in the system. In our situation, it is natural to ask the model checker
to return all evolution traces that lead to inconsistency or
instability. We will show how to compute traces in MCMAS for inconsistency
first and for instability afterwards.

It is usually good in practice to generate the shortest traces for
counterexamples/witness executions in order to decrease the difficulty
of understanding them. To achieve this for our setting, we utilize the
approach of construction of state space in MCMAS. Starting from the
set of initial states $S_0$, MCMAS can compute the state space in the
following manner~\cite{MCMAS}. 

\begin{algorithm} \caption{Compute reachable states} 
\begin{algorithmic}[1]
\STATE{$S \coloneqq \emptyset$; $next\coloneqq S_0$; $q\coloneqq S_0$}

\WHILE {$S\not= q$} 

  \STATE{$S \coloneqq q$; $next\coloneqq Image(next, T)$; $n\coloneqq
    next\setminus S$; $q\coloneqq S\cup next$;}
\ENDWHILE
\STATE{\bf{return} $S$}
\end{algorithmic}
\end{algorithm}
In this algorithm, $S$ is the set of reachable states and $next$,
initialised as $S_0$, is the set of states that their successor states
need to be computed, which is done by the function $Image(next,
T)$. In each iteration, we compute the successors $next'$ of $next$,
and remove from $next'$ the states that have been processed before by
$next'-S$. This iteration continues until no new states can be added
to $S$, i.e., $next'-S=\emptyset$.
We modify the state space generation algorithm to store every intermediate $next$:
in each iteration $i$, we change $next$ to $next_i$. The modified algorithm
is shown in Algorithm~\ref{algo:modified}.
\begin{algorithm} \label{algo:modified}
\caption{Modified state space generation} 
\begin{algorithmic}[1]
\STATE{$S \coloneqq \emptyset$; $next_0\coloneqq S_0$; $q\coloneqq S_0$;
$i\coloneqq 0$}

\WHILE{$S\not= q$} 
  \STATE{$i\coloneqq i+1$}

  \STATE{$S \coloneqq q$; $next_i\coloneqq Image(next_{i-1}, T) \setminus S$;
    $q\coloneqq S\cup next_i$;}
\ENDWHILE
\STATE{\bf{return} $S$, $next_0, \ldots, next_i$}
\end{algorithmic}
\end{algorithm}

\begin{theorem}\label{thm:cex}
A shortest trace leading to an inconsistent state by enabling the
rules $g_1\rightarrow a$ and $g_2\rightarrow \neg a$, can be achieved in the
following steps.
\begin{enumerate}
\item
Starting from $i=0$, we test each $next_i$ to search for the smallest index $k$ such that $next_k \cap g_1 \cap g_2\not=\emptyset$.
\item
We pick up an arbitrary state $s_k$ from $next_k \cap g_1 \cap g_2$ and
compute its predecessor $s_{k-1}$ in $next_{k-1}$ by using the reversed
transition relation $T'$ such that $s_{k-1}\coloneqq s_k\times T'$. If
$s_k$ has multiple predecessors, then we pick up an arbitrary one to
be $s_{k-1}$. In the same way, we compute a predecessor of
$s_{k-1}$ in $next_{k-2}$. This process continues until we find a state $s_0$ in
$next_0$, which is $S_0$. 
\end{enumerate}
\end{theorem}

To find the shortest counterexamples for unstable loops, we need to
identify all such loops first, and for each loop, we test each $next_i$
from $i=0$ if it contains a state in the loop, i.e., if $n_i\cap
S_{loop}\not=\emptyset$, where $S_{loop}$ is the set of states in the
loop. Next we apply the second step in Theorem~\ref{thm:cex} to
generate the shortest trace. Now we focus on how to find all unstable
loops efficiently.

\begin{lemma} \label{lem:scc}
Given a consistent system, none of the unstable loops interfere with each other.
\end{lemma}

\begin{proof}
If the conjunction of two loops is not empty, then there exists a
state such that it has two outgoing transitions, one in each
loop. Hence, this state leads to the occurrence of inconsistency. 
\end{proof}

Due to Lemma~\ref{lem:scc}, finding unstable loops is equivalent to
finding non-trivial strongly connected components (SCCs) when the
system is consistent. There are several SCC identification algorithms
in the literature working on BDD representation of state
spaces~\cite{BGS00,GPP03}. The more efficient one was reported in~\cite{KPQ11}. But
before we apply these algorithms, we could remove states that cannot
reach any unstable loops from the state space in order to speed up the
computation. Those states are identified as $Y$ in the second step of
stability checking in Theorem~\ref{thm:stability}.

\section{Case study} \label{sec:case} 

In this section  we illustrate the use of our
consistency and stability checking techniques on an example scenario which could occur to
a household robot.  
The robot with arms
observes an object rolling across a table. It needs to decide
whether to stop it or allow it to drop off from the table. The
object can be a glass or a light effect. It would be unwise
to stop the object if it is the latter case. The robot may resort to more
accurate sensors to decide how to react. The
model is formalized in a Boolean evolution system
 based on the perception structure in Fig.~\ref{rescyc}.
\begin{enumerate}

\item Feasible sensing possibilities ($B_t$) are:
\begin{itemize}
\item $roll(O)$: object O rolls across table
\item $sensed\_roll(O)$: senses that object O is rolling across table
\item $virt\_real(O)$:  sensed O but there is no real object O
\item $virt\_real\_derived(O)$: derived that light effect moving across table,  there was no real object O sensed
\end{itemize}

In $B_t$, the uncertain sensing events ($U_t$) are $sensed\_roll(O)$
and $virt\_real\_derived(O)$.

\item Action possibilities ($A_t$) are:
\begin{itemize}
\item $stop\_rolling(O)$: stop rolling object by arm
\item $do\_nothing$: remain idle
\end{itemize}

\item Future events predicted ($F_t$) are:
\begin{itemize}
\item $fall(O)$: object O falls 
\item $break(O)$: object O breaks 
\item $useless(O)$: object O is useless
\item $handle(O)$: handling of object O
\item $proper\_observation$: the robot has made the correct observation
\item $proper\_action$: the robot chooses the correct action
\end{itemize}

\item Naive physics rules ($R^P$) are:
\begin{itemize}
\item $\neg  stop\_rolling(O) \land roll(O) \rightarrow fall(O)$: the object will fall if not stopped
\item $fall(O) \rightarrow break(O)$: if the object falls it will break 
\item $stop\_rolling(O) \land roll(O)\rightarrow \neg fall(O) $:  if object is stopped it will not fall
\item $\neg fall(O) \rightarrow \neg break(O)$: if object will not fall then it will not break
\end{itemize}

\item General rules - values and moral consequences rules ($R^B$) are:
\begin{itemize}
\item $virt\_real(O) \land handle(O) \rightarrow wrong\_in\_sensing$
\item $stop\_rolling \rightarrow handle(O)$
\item $break(O) \rightarrow useless(O)$
\item $useless(O) \rightarrow wrong\_in\_action$
\item $\neg break(O) \rightarrow \neg useless(O)$ 
\item $do\_nothing \rightarrow \neg stop\_rolling(O)$
\end{itemize}
 
\end{enumerate}

The robot starts with a simple but fast reasoning cycle by considering
each action individually using observation only. The criteria for
choosing the correct action is to guarantee the following goals.
\begin{itemize}
\item $useless(O)=false$
\item $proper\_action=true$
\item  $proper\_sensing=true$
\end{itemize}

When only one of events $roll(O)$ and $virt\_real(O)$ is true, the
robot can make its decision easily. However, it is difficult to do so
when both events are true. The reasoning process is as follows. 
\begin{enumerate}  
\item {\it Evaluation of action choice 1}:  Goals + $do\_nothing$

This choice results that $proper\_action$ becomes false.

\item {\it Evaluation of action choice 2}:   Goals + $stop\_rolling(O)$

This results inconsistency in the reasoning process as shown in
Fig.~\ref{fig-example}, which demonstrates the evolution of the
value of $proper\_sensing$. 
\begin{figure}[h!]
\centering{\includegraphics[scale=0.7]{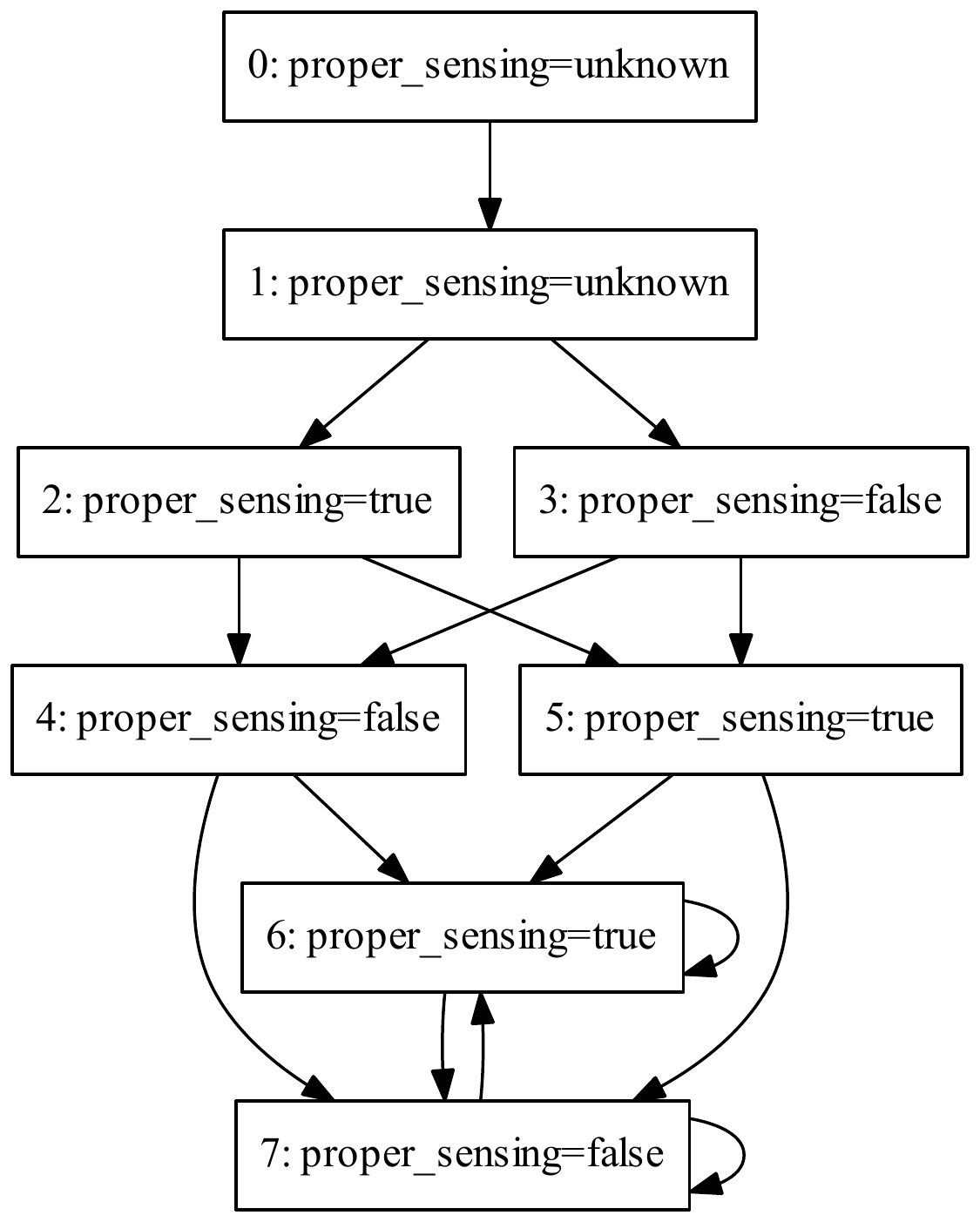}}
\caption{The counterexample showing inconsistency.}
\label{fig-example}
\end{figure}
\end{enumerate}

To resolve the inconsistency, the robot needs to acquire
information from more sensors, which would instantiate the two sensing
events $sensed\_roll(O)$
and $virt\_real\_derived(O)$ in $U_t$ with two extra physical rules. 
\begin{itemize}
\item $\neg sensed\_roll(O) \rightarrow \neg roll(O)$
\item $\neg virt\_real\_derived(O) \rightarrow \neg virt\_real(O)$
\end{itemize}

If these two sensing events do not become true simultaneously, then
the robot can make the correct decision.

Our consistency and stability checking techniques of this kind can be
used in both offline and online modes.  In the online mode,
counterexamples are used to assist the system to acquire more
information, i.e., fixing the uncertain sensing events, or adjusting
the possible actions that can be take, in order to solve inconsistency
or instability problems in a consistency resolution cycle. Our case
study demonstrates an application of the online
mode. Fig.~\ref{rescyc} illustrates the consistency resolution cycle
that can be implemented in agent programs.
\begin{figure}[h!]
\centering{\includegraphics[scale=0.9]{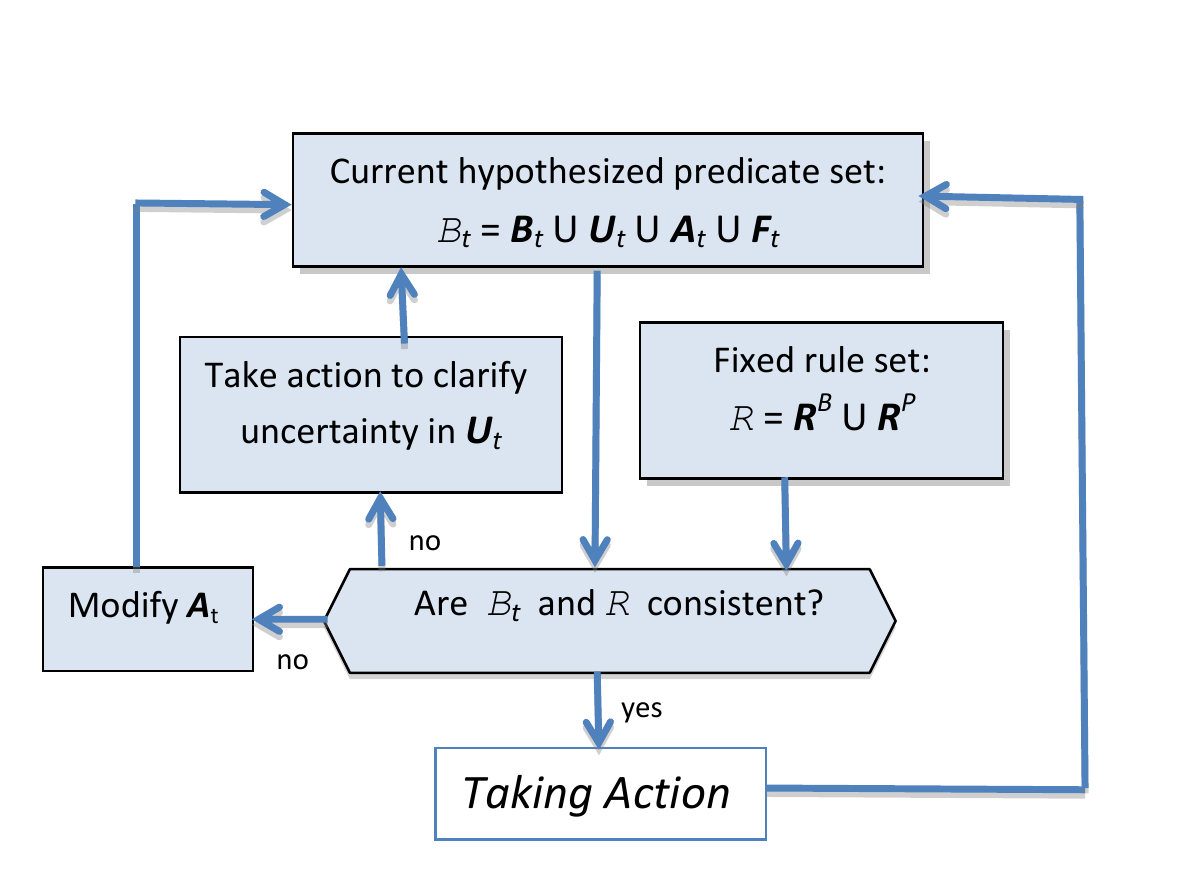}}
\caption{A possible process of inconsistency resolution in agent
  operations. This paper focuses on fast consistency checking with
  counter examples which can guide the modification of $A_t$ and
  actions taken to improve sensory performance to re-evaluate $U_t$. }
\label{rescyc}
\end{figure}

In the offline mode, users can apply the
techniques to their rule-based reasoning system to check consistency
and stability. If a problem is found, users can revise their system
using generated counterexamples as a guidance. For example, If we add
the following extra general rules in the case study,
\begin{itemize}
\item $useless(O) \rightarrow stop\_rolling(O)$
\item $\neg useless(O) \rightarrow \neg stop\_rolling(O)$
\end{itemize}
the robot could be trapped in a circular reasoning process when
$roll(O)$ is true land $virt\_real(O)$ is false because
\begin{equation*}
\begin{split}
stop\_rolling(O) \land roll(O) \longrightarrow & \neg fall(O) \longrightarrow
\neg break(O) \longrightarrow \neg useless(O) \longrightarrow \\ & \neg stop\_rolling(O)
\end{split}
\end{equation*}
and
\begin{equation*}
\begin{split}
\neg stop\_rolling(O) \land roll(O) \longrightarrow & fall(O) \longrightarrow
break(O) \longrightarrow useless(O) \longrightarrow \\ & stop\_rolling(O).
\end{split}
\end{equation*}
This circular reasoning process can be captured by our stability check.

\section{Implementation and performance evaluation} \label{sec:exp}

We have integrated the algorithms in Theorem~\ref{thm:incons1}
and~\ref{thm:reachability} into the model checker MCMAS~\cite{MCMAS}. The
implementation slightly deviated from Theorem~\ref{thm:incons1} in
order to maximize the utility of the internal encoding of the
transition relation in MCMAS. Instead of finding all pair of
conflicting rules, we built a single BDD for each variable $v$. We
first collect all rules $\{ g_1\rightarrow v,\ldots, g_j\rightarrow
v\}$ that set $v$ to $true$, and all rules $\{ g_1'\rightarrow \neg
v,\ldots, g_k'\rightarrow \neg v\}$ that set $v$ to $false$. Second,
we generate a BDD $\mathcal{D}$ representing
\begin{equation}
(g_1\lor \ldots \lor g_j)\land (g_1'\lor \ldots \lor g_k').
\end{equation}
If the BDD is not empty, then there exists a pair of conflicting rules
that might be enabled simultaneously. Searching for a bad state is
done by testing the conjunction of $\mathcal{D}$ and $S$, the set of
reachable states.

To demonstrate
their performance, we applied them to the following example, where
$\mathcal{B}^{known}=\{\mathtt{a}_0, \mathtt{a}_{k, 0},\ldots,
\mathtt{a}_{k, m-1}\}$ and
$\mathcal{B}^{unknown}=\{\mathtt{a}_{1},\ldots,
\mathtt{a}_{m-1}\}$. In the experiment, we fixed $m$ to 32, and
generated a series of models by replicating the group of variables
$\{\mathtt{a}_{k, 0},\ldots, \mathtt{a}_{k, m-1}\}$. In the largest
model, this group has ten copies, i.e., $k$ ranges from $1$ to $10$,
which means the total number of variables is $32 + 32 * 10=352$. Each
variable in $\mathcal{B}^{known}$ requires one BDD variable to encode,
as one BDD variable can represent two values 0 and 1, perfectly
matching Boolean values $false$ and $true$. Each variable in
$\mathcal{B}^{unknown}$ needs two BDD variables because it has three
values. Therefore, the total number of BDD variables in the largest
model is 383.

{\bf Example.} \\
\begin{minipage}{0.52\columnwidth}
\begin{equation*}
\begin{split}
 \mathtt{a}_0 &\rightarrow \mathtt{a}_1\\
 \mathtt{a}_1 &\rightarrow \mathtt{a}_2\\
 ...\\
 \mathtt{a}_{m-2} &\rightarrow \mathtt{a}_{m-1}\\
 \mathtt{a}_{m-1} &\rightarrow \neg \mathtt{a}_0\\
 \neg \mathtt{a}_0 &\rightarrow \neg \mathtt{a}_1\\
 ...\\
 \neg \mathtt{a}_{m-2} &\rightarrow \neg \mathtt{a}_{m-1}\\ 
 \neg \mathtt{a}_{m-1} &\rightarrow \mathtt{a}_0\\ 
\end{split}
\end{equation*}
\end{minipage}
\begin{minipage}{0.4\columnwidth}
\begin{equation*}
\begin{split}
 \mathtt{a}_{k, 0} &\rightarrow \mathtt{a}_{k, 1}\\
 \mathtt{a}_{k, 1} &\rightarrow \mathtt{a}_{k, 2}\\
 ...\\
 \mathtt{a}_{k, m-2} &\rightarrow \mathtt{a}_{k, m-1}\\
 \mathtt{a}_{k, m-1} &\rightarrow \neg \mathtt{a}_{k, 0}\\
 \neg \mathtt{a}_{k, 0} &\rightarrow \neg \mathtt{a}_{k, 1}\\
 ...\\
 \neg \mathtt{a}_{k, m-2} &\rightarrow \neg \mathtt{a}_{k, m-1}\\ 
 \neg \mathtt{a}_{k, m-1} &\rightarrow \mathtt{a}_{k, 0}\\ 
\end{split}
\end{equation*}
\end{minipage}

The experimental results are listed in Table~\ref{table-exp}. For each
model, we present the number of variables and corresponding BDD
variables in parentheses, as well as the number of reachable
states. The time (in second) spent on checking consistency and
stability via the CTL formulae~\ref{eqn:incons} and~\ref{eqn:ctl} are
shown in the two columns in the middle, and the time for the direct
algorithms in Theorem~\ref{thm:incons1} and~\ref{thm:reachability} are
given in the last two columns. The results clearly demonstrates the
huge advantage of using our stability checking algorithm. The
performance of our consistency checking algorithm is also excellent,
given the fact that the CTL formula~\ref{eqn:incons} is quite
efficient already. Note that the time
  spent on building BDD $\mathcal{D}$ for each variable is shown to be below 1ms
  in the penultimate column of the table.

\begin{table}[h] \centering
\caption{\label{table-exp} Experimental results.}
\setlength{\tabcolsep}{1pt}
\begin{tabular}{|c|r||r|r||r|r|}
\hline
\multirow{3}{1.4cm}{\centering Num of \\ variables} &
\multirow{3}{1.4cm}{\centering Num of \\ states} &
\multicolumn{2}{c||}{CTL formulae} &
\multicolumn{2}{c|}{Direct algorithms} \\
\cline{3-6}
& & \multicolumn{1}{c|}{Consistency} & \multicolumn{1}{c||}{Stability}  &
 \multicolumn{1}{c|}{Consistency} & \multicolumn{1}{c|}{Stability}  \\
& & \multicolumn{1}{c|}{time (s)} & \multicolumn{1}{c||}{time (s)} &
 \multicolumn{1}{c|}{time (s)} & \multicolumn{1}{c|}{time (s)} \\
\hline
64 (95) & 5.41166e+11 & 0.049 & 1.884 & $< 0.001$ & 0.001 \\
\hline
96 (127)  & 2.32429e+21 & 0.128 & 4.773 & $< 0.001$ & 0.002 \\
\hline
128 (159) & 9.98275e+30 & 0.248 & 9.073 & $< 0.001$ & 0.003 \\
\hline
160 (191) & 4.28756e+40 & 0.31 & 10.381 & $< 0.001$ & 0.002 \\
\hline
192 (223) & 1.84149e+50 & 0.547 & 19.766 & $< 0.001$ & 0.003 \\
\hline
224 (255) & 7.90915e+59 & 0.867 & 29.341 & $< 0.001$ & 0.008 \\
\hline
256 (287) & 3.39695e+69 & 1.154 & 38.216 & $< 0.001$ & 0.01 \\
\hline
288 (319) & 1.45898e+79 & 0.571 & 19.169 & $< 0.001$ & 0.066 \\
\hline
320 (351) & 6.26627e+88 & 0.849 & 29.308 & $< 0.001$ & 0.062 \\
\hline
352 (383) & 2.69134e+98 & 2.242 & 73.112 & $< 0.001$ & 0.022 \\

\hline
\end{tabular}
\end{table}

\section{Discussion on interleaving semantics} \label{sec:interleaving}

Although synchronous semantics have been applied broadly in practice,
a differently semantics, {\em interleaving semantics}, still finds its
usefulness in case of limited processing power. Interleaving means
that only one enabled rule, which is usually chosen randomly, is
processed at a time.


\begin{definition}[Interleaving semantics]
  The new valuation $\overline{\mathcal{B}}'$ under interleaving
  semantics is the result of applying a rule $r$ in
  $\mathcal{R}|_{\overline{\mathcal{B}}}$ to
  $\overline{\mathcal{B}}$. The rule $r$ is chosen
  non-deterministically. That is, every new value of $b$ in
  $\overline{\mathcal{B}}'$ is defined as follows.
\begin{equation*} \label{eqn:interleaving}
\overline{\mathcal{B}}'(b) = \left\{
 \begin{array}{l l}
   true & \quad \text{if $r = g \rightarrow b$,}\\
   false & \quad \text{if $r = g \rightarrow \neg b$,}\\
   \overline{\mathcal{B}}(b) & \quad \text{otherwise.}
 \end{array} \right.
\end{equation*}
\end{definition}

Under the relaxed inconsistency conditions, a system is guaranteed to
be consistent, if at any given time, only one rule is
processed. Therefore, the first inconsistent condition is not
satisfied any more. However, the interleaving semantics possesses different
characteristics during stability checking. A stable system under
the synchronous semantics may become unstable. Let us re-examine Example 2
using the interleaving semantics. We can construct a path that visits
unstable states as follows. 
For valuation $a=true$, we have
$1???\longrightarrow 11?? \longrightarrow 11?1 \longrightarrow 1101 
\longrightarrow 0101 \longrightarrow 0001 \longrightarrow 0011 
\longrightarrow 0111 \longrightarrow 0101 \longrightarrow\cdots$. The
infinite loop 
\[0101 \longrightarrow 0001 \longrightarrow 0011 
\longrightarrow 0111 \longrightarrow 0101\] makes the system unstable.

However, the infinite loop is quite special in that the rule
\[\mathtt{b} \land \mathtt{c} \rightarrow \neg \mathtt{d}\] is enabled
infinitely often in state $0111$, which is the beginning of the
unstable loop. In practice, such infinite loops rarely happen because of
randomness of choice. Once this rule is executed, the unstable loop is
interrupted, and the system becomes stable. This observation leads to
the introduction of fairness into stability checking. 

Fairness~\cite{BK08} has been studied and applied to many temporal logic-based
verification, including both CTL and LTL. Various types of fairness
constraints have been brought up. Among them, the most popular ones
are {\em unconditional}, {\em strong} and {\em weak} fairness. In this
section, {\em strong} fairness is sufficient to exclude the
above unrealistic/unfair paths. 

\begin{definition}[Strong fairness]
Strong fairness under interleaving semantics requires that,
in every infinite path, an evolution rule has to be executed
infinitely often if it is enabled infinitely often. For transition
systems, strong fairness is composed of a set of fairness constraints,
written as 
\begin{equation}
\bigwedge_i (\Box\Diamond \Phi_i \implies \Box\Diamond \Psi_i),\label{eqn:fair}
\end{equation}
where each constraints $\Box\Diamond \Phi_i \implies \Box\Diamond
\Psi_i$ specifies that if $\Phi_i$ occurs infinitely often, then
$\Psi_i$ has to occurs infinitely often as well.
\end{definition}

Strong fairness rules out unrealistic evolution paths, where some
enabled rules are consistently ignored. Therefore, it allows more
systems   evaluated as stable. For Example 2, we only need one
fairness constraint:
\[ \Box\Diamond (\mathtt{B}_b \land \mathtt{B}_c) \implies
\Box\Diamond \neg\mathtt{B}_d.\]
This example suggests that the generation of a fairness constraint from a
rule can be straightforward, which can be achieved by following the
syntactic form of the rule.

However, strong fairness still cannot prevent some stable
system under synchronous semantics from being unstable. The following
example demonstrates an unstable system under strong
fairness. In this example, $\mathcal{B}^{known}=\{a\}$ and
$\mathcal{B}^{unknown}=\{b,c,d,e\}$

{\bf Example 3.}
\begin{equation*}
\begin{split}
 \mathtt{a} &\rightarrow \mathtt{b} \land \mathtt{d} \land \mathtt{e}\\
 \mathtt{b} \land \mathtt{d} &\rightarrow \mathtt{c} \land \neg \mathtt{a}\\
 \mathtt{c} \land \mathtt{d} &\rightarrow \neg \mathtt{b}\\
 \neg \mathtt{b} \land \mathtt{d} &\rightarrow \neg \mathtt{c}\\
 \neg \mathtt{c} \land \mathtt{d} &\rightarrow  \mathtt{b}\\
 \mathtt{b} \land \mathtt{c} \land \mathtt{e} &\rightarrow \neg \mathtt{d}\\
 \neg \mathtt{b} \land \neg \mathtt{c} &\rightarrow \neg \mathtt{e}
\end{split}
\end{equation*}

For the initial valuation $a=true$, we have
$1????\longrightarrow 11??? \longrightarrow 11?1? \longrightarrow 11?11 
\longrightarrow 11111 \longrightarrow 01111 \longrightarrow 00111 
\longrightarrow 00011 \longrightarrow 00010 \longrightarrow 01010
\longrightarrow 01110 \longrightarrow 00110 \longrightarrow 00010 \cdots$. The
unstable loop 
\[00010 \longrightarrow 01010
\longrightarrow 01110 \longrightarrow 00110 \longrightarrow 00010\]
cannot be broken because the only rule that can break it, i.e., 
\[ \mathtt{b} \land \mathtt{c} \land \mathtt{e} \rightarrow \neg \mathtt{d}\]
 is disabled in state $00010$.

 Enforcing strong fairness in the verification of CTL formulas can be
 transformed into finding strongly connected components (SCCs), which
 in turn can be implemented using graphic algorithms~\cite{BK08}. As
 we do not consider inconsistency for the interleaving semantics, the
 stability can be checked by LTL model checkers, such as
 SPIN~\cite{spin} and NuSMV. The verification of an LTL formula $f$ under strong
 fairness can be achieved directly by checking a combined LTL formula 
\begin{equation}
fair \implies f,
\end{equation}
where $fair$ is of the form of Formula (\ref{eqn:fair}). For stability
checking the $f$ is taken as in~(\ref{eqn:ltl}). Note that the algorithm
in Theorem~\ref{thm:reachability} does not work here because multiple
successor states do not mean inconsistency any more. SPIN
uses explicit model checking techniquesfor the verification. It
requires that every initial valuation has to be enumerated explicitly,
which is very inefficient. NuSMV adopts the method in~\cite{CGH97} to check LTL formulae
symbolically via BDDs, which can be more efficient for our purpose.

Now the question is how we identify rules that need to be guaranteed
for execution by strong fairness. Human guidance on the selection of
rules would be ideal. When it is not available, we need to find a
solution to allow automatic selection.
A simple method is to put all
rules under the protection of fairness. This solution does not request the
modification of an existing model checker. However, it only works for a
small rule set. A large number of rules would render $fair$ a large
LTL formula containing  equal number of constraints as the number  of
rules. 

An alternative solution utilises a sequence of verification to search
for a fair unstable loop. Starting with no fairness constraints, we
check Formula~(\ref{eqn:ltl}) solely on the converted transition
system. If the result is $false$, which means the system may be
unstable, then we ask the model checker to generate a counterexample,
i.e. an unstable loop. We examine each state in the loop to look for
enabled rules that are never executed in the loop. If no such rule is
found, then the system is unstable under strong fairness. Otherwise,
we put the unfairly treated rules into $fair$ and re-start the
verification. This process is carried out iteratively until the system
is proven to be stable or unstable under strong fairness. Although the
idea of this solution is not complex, its implementation requires to
build an extension of a model checker, e.g., NuSMV, which is not a
trivial task. Further, its performance would be degraded when the
number of iterations increases.

\section{Conclusion} \label{sec:concl}

This paper has solved the problem of efficiency for logical
consistency checks of robots by adopting symbolic model checking based
on binary decision diagrams.  In addition to specifying stability and
consistency as CTL and LTL formulas, also efficient symbolic
algorithms have been applied to speed up consistency and stability
checking.  Timely decision making by robots can be vital in many
safety critical applications. The most basic task they need to check is
the consistency of their assumptions and their rules. Speed of computation hence
affects quality and safety of robots. As a first step towards application of our approach,
we have embedded it within the framework LISA~\cite{ecc16lisa,taros16} 
for reasoning of robots.

Further direct use of the techniques is in rule-based reasoning
systems before they are deployed on robots.  The counter-examples, which
can generated by the techniques  presented, can demonstrate the reasons for possible
violation, which can help software developers revising their
designs. Sometimes it can be time-consuming to modify the design of a
complex system, possibly the violation is tolerable or is very rare
during run-time. In these cases counter-examples can be used as a
guide to correct the reasoning of a robot while in action.
 
Future work in this research area can target the implementation of our
approach in practical programming ~\cite{jsce_000,review2011}, and to
aid finding solutions to making inconsistent/unstable systems
consistent and stable. Based on the results an iterative design
process can be defined to enable a programmer to control an agent's
decision making.  Our plans are also to integrate consistency checks in LISA~\cite{ecc16lisa,taros16}
 into the control code of Unmanned Aerial Vehicles (UAVs) and Unmanned Ground Vehicles (UGVs) for 
practical application.

\section*{Acknowledgements}
This work was supported by the EPSRC project EP/J011894/2.








\end{document}